\documentclass[conference,10pt,twocolumn,letter]{IEEEtran}
\usepackage{times}
\usepackage{helvet}
\usepackage{courier}
\usepackage{nicefrac}  
\usepackage[utf8]{inputenc} 
\usepackage[T1]{fontenc}    
\usepackage{url}            
\usepackage{graphicx}
\usepackage{nicefrac}       
\usepackage{amssymb}
\usepackage{amsmath}
\usepackage{amsthm}
\usepackage{paralist}
\usepackage{subfigure}
\usepackage{setspace}
\usepackage{booktabs}
\usepackage{algorithm}
\usepackage{microtype}
\usepackage{cite}
\usepackage{balance}
\usepackage{algpseudocode}

\usepackage{amssymb}
\usepackage{amsfonts}
\usepackage{mathrsfs}
\usepackage{xspace}
\usepackage{bm}
\usepackage{upgreek}

\newcommand{\safemath}[2]{\newcommand{#1}{\ensuremath{#2}\xspace}}

\safemath{\bma}{\mathbf{a}}
\safemath{\bmb}{\mathbf{b}}
\safemath{\bmc}{\mathbf{c}}
\safemath{\bmd}{\mathbf{d}}
\safemath{\bme}{\mathbf{e}}
\safemath{\bmf}{\mathbf{f}}
\safemath{\bmg}{\mathbf{g}}
\safemath{\bmh}{\mathbf{h}}
\safemath{\bmi}{\mathbf{i}}
\safemath{\bmj}{\mathbf{j}}
\safemath{\bmk}{\mathbf{k}}
\safemath{\bml}{\mathbf{l}}
\safemath{\bmm}{\mathbf{m}}
\safemath{\bmn}{\mathbf{n}}
\safemath{\bmo}{\mathbf{o}}
\safemath{\bmp}{\mathbf{p}}
\safemath{\bmq}{\mathbf{q}}
\safemath{\bmr}{\mathbf{r}}
\safemath{\bms}{\mathbf{s}}
\safemath{\bmt}{\mathbf{t}}
\safemath{\bmu}{\mathbf{u}}
\safemath{\bmv}{\mathbf{v}}
\safemath{\bmw}{\mathbf{w}}
\safemath{\bmx}{\mathbf{x}}
\safemath{\bmy}{\mathbf{y}}
\safemath{\bmz}{\mathbf{z}}
\safemath{\bmzero}{\mathbf{0}}
\safemath{\bmone}{\mathbf{1}}

\bmdefine{\biad}{a}
\bmdefine{\bibd}{b}
\bmdefine{\bicd}{c}
\bmdefine{\bidd}{d}
\bmdefine{\bied}{e}
\bmdefine{\bifd}{f}
\bmdefine{\bigd}{g}
\bmdefine{\bihd}{h}
\bmdefine{\biid}{i}
\bmdefine{\bijd}{j}
\bmdefine{\bikd}{k}
\bmdefine{\bild}{l}
\bmdefine{\bimd}{m}
\bmdefine{\bind}{n}
\bmdefine{\biod}{o}
\bmdefine{\bipd}{p}
\bmdefine{\biqd}{q}
\bmdefine{\bird}{r}
\bmdefine{\bisd}{s}
\bmdefine{\bitd}{t}
\bmdefine{\biud}{u}
\bmdefine{\bivd}{v}
\bmdefine{\biwd}{w}
\bmdefine{\bixd}{x}
\bmdefine{\biyd}{y}
\bmdefine{\bizd}{z}

\bmdefine{\bixid}{\xi}
\bmdefine{\bilambdad}{\lambda}
\bmdefine{\bimud}{\mu}
\bmdefine{\bithetad}{\theta}
\bmdefine{\biphid}{\phi}
\bmdefine{\bideltad}{\delta}

\safemath{\bmia}{\biad}
\safemath{\bmib}{\bibd}
\safemath{\bmic}{\bicd}
\safemath{\bmid}{\bidd}
\safemath{\bmie}{\bied}
\safemath{\bmif}{\bifd}
\safemath{\bmig}{\bigd}
\safemath{\bmih}{\bihd}
\safemath{\bmii}{\biid}
\safemath{\bmij}{\bijd}
\safemath{\bmik}{\bikd}
\safemath{\bmil}{\bild}
\safemath{\bmim}{\bimd}
\safemath{\bmin}{\bind}
\safemath{\bmio}{\biod}
\safemath{\bmip}{\bipd}
\safemath{\bmiq}{\biqd}
\safemath{\bmir}{\bird}
\safemath{\bmis}{\bisd}
\safemath{\bmit}{\bitd}
\safemath{\bmiu}{\biud}
\safemath{\bmiv}{\bivd}
\safemath{\bmiw}{\biwd}
\safemath{\bmix}{\bixd}
\safemath{\bmiy}{\biyd}
\safemath{\bmiz}{\bizd}

\safemath{\bmxi}{\bixid}
\safemath{\bmlambda}{\bilambdad}
\safemath{\bmmu}{\bimud}
\safemath{\bmtheta}{\bithetad}
\safemath{\bmphi}{\biphid}
\safemath{\bmdelta}{\bideltad}

\safemath{\bA}{\mathbf{A}}
\safemath{\bB}{\mathbf{B}}
\safemath{\bC}{\mathbf{C}}
\safemath{\bD}{\mathbf{D}}
\safemath{\bE}{\mathbf{E}}
\safemath{\bF}{\mathbf{F}}
\safemath{\bG}{\mathbf{G}}
\safemath{\bH}{\mathbf{H}}
\safemath{\bI}{\mathbf{I}}
\safemath{\bJ}{\mathbf{J}}
\safemath{\bK}{\mathbf{K}}
\safemath{\bL}{\mathbf{L}}
\safemath{\bM}{\mathbf{M}}
\safemath{\bN}{\mathbf{N}}
\safemath{\bO}{\mathbf{O}}
\safemath{\bP}{\mathbf{P}}
\safemath{\bQ}{\mathbf{Q}}
\safemath{\bR}{\mathbf{R}}
\safemath{\bS}{\mathbf{S}}
\safemath{\bT}{\mathbf{T}}
\safemath{\bU}{\mathbf{U}}
\safemath{\bV}{\mathbf{V}}
\safemath{\bW}{\mathbf{W}}
\safemath{\bX}{\mathbf{X}}
\safemath{\bY}{\mathbf{Y}}
\safemath{\bZ}{\mathbf{Z}}

\safemath{\bZero}{\mathbf{0}}
\safemath{\bOne}{\mathbf{1}}
\safemath{\bDelta}{\mathbf{\Delta}}
\safemath{\bLambda}{\mathbf{\Lambda}}
\safemath{\bPhi}{\mathbf{\Upphi}}
\safemath{\bSigma}{\mathbf{\Upsigma}}
\safemath{\bOmega}{\mathbf{\Upomega}}
\safemath{\bTheta}{\mathbf{\Uptheta}}

\bmdefine{\biAd}{A}
\bmdefine{\biBd}{B}
\bmdefine{\biCd}{C}
\bmdefine{\biDd}{D}
\bmdefine{\biEd}{E}
\bmdefine{\biFd}{F}
\bmdefine{\biGd}{G}
\bmdefine{\biHd}{H}
\bmdefine{\biId}{I}
\bmdefine{\biJd}{J}
\bmdefine{\biKd}{K}
\bmdefine{\biLd}{L}
\bmdefine{\biMd}{M}
\bmdefine{\biOd}{N}
\bmdefine{\biPd}{O}
\bmdefine{\biQd}{P}
\bmdefine{\biRd}{R}
\bmdefine{\biSd}{S}
\bmdefine{\biTd}{T}
\bmdefine{\biUd}{U}
\bmdefine{\biVd}{V}
\bmdefine{\biWd}{W}
\bmdefine{\biXd}{X}
\bmdefine{\biYd}{Y}
\bmdefine{\biZd}{Z}

\bmdefine{\biDelta}{\Delta}
\bmdefine{\biLambda}{\Lambda}
\bmdefine{\biPhi}{\Phi}
\bmdefine{\biSigma}{\Sigma}
\bmdefine{\biOmega}{\Omega}
\bmdefine{\biTheta}{\Theta}

\safemath{\bimA}{\biAd}
\safemath{\bimB}{\biBd}
\safemath{\bimC}{\biCd}
\safemath{\bimD}{\biDd}
\safemath{\bimE}{\biEd}
\safemath{\bimF}{\biFd}
\safemath{\bimG}{\biGd}
\safemath{\bimH}{\biHd}
\safemath{\bimI}{\biId}
\safemath{\bimJ}{\biJd}
\safemath{\bimK}{\biKd}
\safemath{\bimL}{\biLd}
\safemath{\bimM}{\biMd}
\safemath{\bimN}{\biNd}
\safemath{\bimO}{\biOd}
\safemath{\bimP}{\biPd}
\safemath{\bimQ}{\biQd}
\safemath{\bimR}{\biRd}
\safemath{\bimS}{\biSd}
\safemath{\bimT}{\biTd}
\safemath{\bimU}{\biUd}
\safemath{\bimV}{\biVd}
\safemath{\bimW}{\biWd}
\safemath{\bimX}{\biXd}
\safemath{\bimY}{\biYd}
\safemath{\bimZ}{\biZd}

\safemath{\bimDelta}{\biDelta}
\safemath{\bimLambda}{\biLambda}
\safemath{\bimPhi}{\biPhi}
\safemath{\bimSigma}{\biSigma}
\safemath{\bimOmega}{\biOmega}
\safemath{\bimTheta}{\biTheta}

\safemath{\setA}{\mathcal{A}}
\safemath{\setB}{\mathcal{B}}
\safemath{\setC}{\mathcal{C}}
\safemath{\setD}{\mathcal{D}}
\safemath{\setE}{\mathcal{E}}
\safemath{\setF}{\mathcal{F}}
\safemath{\setG}{\mathcal{G}}
\safemath{\setH}{\mathcal{H}}
\safemath{\setI}{\mathcal{I}}
\safemath{\setJ}{\mathcal{J}}
\safemath{\setK}{\mathcal{K}}
\safemath{\setL}{\mathcal{L}}
\safemath{\setM}{\mathcal{M}}
\safemath{\setN}{\mathcal{N}}
\safemath{\setO}{\mathcal{O}}
\safemath{\setP}{\mathcal{P}}
\safemath{\setQ}{\mathcal{Q}}
\safemath{\setR}{\mathcal{R}}
\safemath{\setS}{\mathcal{S}}
\safemath{\setT}{\mathcal{T}}
\safemath{\setU}{\mathcal{U}}
\safemath{\setV}{\mathcal{V}}
\safemath{\setW}{\mathcal{W}}
\safemath{\setX}{\mathcal{X}}
\safemath{\setY}{\mathcal{Y}}
\safemath{\setZ}{\mathcal{Z}}
\safemath{\emptySet}{\varnothing}

\safemath{\colA}{\mathscr{A}}
\safemath{\colB}{\mathscr{B}}
\safemath{\colC}{\mathscr{C}}
\safemath{\colD}{\mathscr{D}}
\safemath{\colE}{\mathscr{E}}
\safemath{\colF}{\mathscr{F}}
\safemath{\colG}{\mathscr{G}}
\safemath{\colH}{\mathscr{H}}
\safemath{\colI}{\mathscr{I}}
\safemath{\colJ}{\mathscr{J}}
\safemath{\colK}{\mathscr{K}}
\safemath{\colL}{\mathscr{L}}
\safemath{\colM}{\mathscr{M}}
\safemath{\colN}{\mathscr{N}}
\safemath{\colO}{\mathscr{O}}
\safemath{\colP}{\mathscr{P}}
\safemath{\colQ}{\mathscr{Q}}
\safemath{\colR}{\mathscr{R}}
\safemath{\colS}{\mathscr{S}}
\safemath{\colT}{\mathscr{T}}
\safemath{\colU}{\mathscr{U}}
\safemath{\colV}{\mathscr{V}}
\safemath{\colW}{\mathscr{W}}
\safemath{\colX}{\mathscr{X}}
\safemath{\colY}{\mathscr{Y}}
\safemath{\colZ}{\mathscr{Z}}

\safemath{\opA}{\mathbb{A}}
\safemath{\opB}{\mathbb{B}}
\safemath{\opC}{\mathbb{C}}
\safemath{\opD}{\mathbb{D}}
\safemath{\opE}{\mathbb{E}}
\safemath{\opF}{\mathbb{F}}
\safemath{\opG}{\mathbb{G}}
\safemath{\opH}{\mathbb{H}}
\safemath{\opI}{\mathbb{I}}
\safemath{\opJ}{\mathbb{J}}
\safemath{\opK}{\mathbb{K}}
\safemath{\opL}{\mathbb{L}}
\safemath{\opM}{\mathbb{M}}
\safemath{\opN}{\mathbb{N}}
\safemath{\opO}{\mathbb{O}}
\safemath{\opP}{\mathbb{P}}
\safemath{\opQ}{\mathbb{Q}}
\safemath{\opR}{\mathbb{R}}
\safemath{\opS}{\mathbb{S}}
\safemath{\opT}{\mathbb{T}}
\safemath{\opU}{\mathbb{U}}
\safemath{\opV}{\mathbb{V}}
\safemath{\opW}{\mathbb{W}}
\safemath{\opX}{\mathbb{X}}
\safemath{\opY}{\mathbb{Y}}
\safemath{\opZ}{\mathbb{Z}}
\safemath{\opZero}{\mathbb{O}}
\safemath{\identityop}{\opI}

\safemath{\veca}{\bma}
\safemath{\vecb}{\bmb}
\safemath{\vecc}{\bmc}
\safemath{\vecd}{\bmd}
\safemath{\vece}{\bme}
\safemath{\vecf}{\bmf}
\safemath{\vecg}{\bmg}
\safemath{\vech}{\bmh}
\safemath{\veci}{\bmi}
\safemath{\vecj}{\bmj}
\safemath{\veck}{\bmk}
\safemath{\vecl}{\bml}
\safemath{\vecm}{\bmm}
\safemath{\vecn}{\bmn}
\safemath{\veco}{\bmo}
\safemath{\vecp}{\bmp}
\safemath{\vecq}{\bmq}
\safemath{\vecr}{\bmr}
\safemath{\vecs}{\bms}
\safemath{\vect}{\bmt}
\safemath{\vecu}{\bmu}
\safemath{\vecv}{\bmv}
\safemath{\vecw}{\bmw}
\safemath{\vecx}{\bmx}
\safemath{\vecy}{\bmy}
\safemath{\vecz}{\bmz}

\safemath{\veczero}{\bmzero}
\safemath{\vecone}{\bmone}
\safemath{\vecxi}{\bmxi}
\safemath{\veclambda}{\bmlambda}
\safemath{\vecmu}{\bmmu}
\safemath{\vectheta}{\bmtheta}
\safemath{\vecphi}{\bmphi}
\safemath{\vecdelta}{\bmdelta}

\safemath{\matA}{\bA}
\safemath{\matB}{\bB}
\safemath{\matC}{\bC}
\safemath{\matD}{\bD}
\safemath{\matE}{\bE}
\safemath{\matF}{\bF}
\safemath{\matG}{\bG}
\safemath{\matH}{\bH}
\safemath{\matI}{\bI}
\safemath{\matJ}{\bJ}
\safemath{\matK}{\bK}
\safemath{\matL}{\bL}
\safemath{\matM}{\bM}
\safemath{\matN}{\bN}
\safemath{\matO}{\bO}
\safemath{\matP}{\bP}
\safemath{\matQ}{\bQ}
\safemath{\matR}{\bR}
\safemath{\matS}{\bS}
\safemath{\matT}{\bT}
\safemath{\matU}{\bU}
\safemath{\matV}{\bV}
\safemath{\matW}{\bW}
\safemath{\matX}{\bX}
\safemath{\matY}{\bY}
\safemath{\matZ}{\bZ}
\safemath{\matzero}{\bmzero}

\safemath{\matDelta}{\bDelta}
\safemath{\matLambda}{\bLambda}
\safemath{\matPhi}{\bPhi}
\safemath{\matSigma}{\bSigma}
\safemath{\matOmega}{\bOmega}
\safemath{\matTheta}{\bTheta}

\safemath{\matidentity}{\matI}
\safemath{\matone}{\matO}

\safemath{\rnda}{A}
\safemath{\rndb}{B}
\safemath{\rndc}{C}
\safemath{\rndd}{D}
\safemath{\rnde}{E}
\safemath{\rndf}{F}
\safemath{\rndg}{G}
\safemath{\rndh}{H}
\safemath{\rndi}{I}
\safemath{\rndj}{J}
\safemath{\rndk}{K}
\safemath{\rndl}{L}
\safemath{\rndm}{M}
\safemath{\rndn}{N}
\safemath{\rndo}{O}
\safemath{\rndp}{P}
\safemath{\rndq}{Q}
\safemath{\rndr}{R}
\safemath{\rnds}{S}
\safemath{\rndt}{T}
\safemath{\rndu}{U}
\safemath{\rndv}{V}
\safemath{\rndw}{W}
\safemath{\rndx}{X}
\safemath{\rndy}{Y}
\safemath{\rndz}{Z}

\safemath{\rveca}{\bimA}
\safemath{\rvecb}{\bimB}
\safemath{\rvecc}{\bimC}
\safemath{\rvecd}{\bimD}
\safemath{\rvece}{\bimE}
\safemath{\rvecf}{\bimF}
\safemath{\rvecg}{\bimG}
\safemath{\rvech}{\bimH}
\safemath{\rveci}{\bimI}
\safemath{\rvecj}{\bimJ}
\safemath{\rveck}{\bimK}
\safemath{\rvecl}{\bimL}
\safemath{\rvecm}{\bimM}
\safemath{\rvecn}{\bimN}
\safemath{\rveco}{\bomO}
\safemath{\rvecp}{\bimP}
\safemath{\rvecq}{\bimQ}
\safemath{\rvecr}{\bimR}
\safemath{\rvecs}{\bimS}
\safemath{\rvect}{\bimT}
\safemath{\rvecu}{\bimU}
\safemath{\rvecv}{\bimV}
\safemath{\rvecw}{\bimW}
\safemath{\rvecx}{\bimX}
\safemath{\rvecy}{\bimY}
\safemath{\rvecz}{\bimZ}

\safemath{\rvecxi}{\bmxi}
\safemath{\rveclambda}{\bmlambda}
\safemath{\rvecmu}{\bmmu}
\safemath{\rvectheta}{\bmtheta}
\safemath{\rvecphi}{\bmphi}

\safemath{\rmatA}{\bimA}
\safemath{\rmatB}{\bimB}
\safemath{\rmatC}{\bimC}
\safemath{\rmatD}{\bimD}
\safemath{\rmatE}{\bimE}
\safemath{\rmatF}{\bimF}
\safemath{\rmatG}{\bimG}
\safemath{\rmatH}{\bimH}
\safemath{\rmatI}{\bimI}
\safemath{\rmatJ}{\bimJ}
\safemath{\rmatK}{\bimK}
\safemath{\rmatL}{\bimL}
\safemath{\rmatM}{\bimM}
\safemath{\rmatN}{\bimN}
\safemath{\rmatO}{\bimO}
\safemath{\rmatP}{\bimP}
\safemath{\rmatQ}{\bimQ}
\safemath{\rmatR}{\bimR}
\safemath{\rmatS}{\bimS}
\safemath{\rmatT}{\bimT}
\safemath{\rmatU}{\bimU}
\safemath{\rmatV}{\bimV}
\safemath{\rmatW}{\bimW}
\safemath{\rmatX}{\bimX}
\safemath{\rmatY}{\bimY}
\safemath{\rmatZ}{\bimZ}

\safemath{\rmatDelta}{\bimDelta}
\safemath{\rmatLambda}{\bimLambda}
\safemath{\rmatPhi}{\bimPhi}
\safemath{\rmatSigma}{\bimSigma}
\safemath{\rmatOmega}{\bimOmega}
\safemath{\rmatTheta}{\bimTheta}

\usepackage{amssymb}
\usepackage{amsfonts}
\usepackage{mathrsfs}
\usepackage{xspace}
\usepackage{bm}
\usepackage{fancyref}
\usepackage{textcomp}

\usepackage{multirow}
\usepackage{stmaryrd}

\newenvironment{textbmatrix}{	\setlength{\arraycolsep}{2.5pt}%
								\big[\begin{matrix}}{\end{matrix}\big]%
								\raisebox{0.08ex}{\vphantom{M}}}

\def\be{\begin{equation}}
\def\ee{\end{equation}}
\def\een{\nonumber \end{equation}}
\def\mat{\begin{bmatrix}}
\def\emat{\end{bmatrix}}
\def\btm{\begin{textbmatrix}}
\def\etm{\end{textbmatrix}}

\def\ba#1\ea{\begin{align}#1\end{align}}
\def\bas#1\eas{\begin{align*}#1\end{align*}}
\def\bs#1\es{\begin{split}#1\end{split}} 
\def\bg#1\eg{\begin{gather}#1\end{gather}}
\def\bml#1\eml{\begin{multline}#1\end{multline}}
\def\bi#1\ei{\begin{itemize}#1\end{itemize}}

\newcommand{\lefto}{\mathopen{}\left}

\DeclareMathOperator{\tr}{tr}				
\DeclareMathOperator{\sign}{sign}			
\DeclareMathOperator{\Exop}{\opE}			

\newcommand{\Ex}[2]{\ensuremath{\Exop_{#1}\lefto[#2\right]}} 	
\newcommand{\abs}[1]{\lefto\lvert#1\right\rvert}		

\safemath{\dirac}{\delta}					
\safemath{\krond}{\dirac}					

\safemath{\upto}{\uparrow}
\safemath{\downto}{\downarrow}
\safemath{\iu}{j}							
\safemath{\ev}{\lambda}						
\safemath{\hilseqspace}{l^{2}}				
\newcommand{\banachfunspace}[1]{\setL^{#1}}	
\safemath{\hilfunspace}{\banachfunspace{2}}	

\safemath{\SNR}{\textsf{SNR}} 				
\safemath{\PAR}{\textsf{PAR}} 				
\safemath{\No}{N_0}							
\safemath{\Es}{E_s}							
\safemath{\Eb}{E_b}							
\safemath{\EbNo}{\frac{\Eb}{\No}}
\safemath{\EsNo}{\frac{\Es}{\No}}

\DeclareMathOperator{\CHop}{\ensuremath{\opH}} 
\safemath{\tvir}{\rndh_{\CHop}}				
\safemath{\tvtf}{\rndl_{\CHop}}				
\safemath{\spf}{\rnds_{\CHop}}				
\safemath{\bff}{H_{\CHop}}					

\safemath{\ircf}{r_{h}}						
\safemath{\tftvcf}{r_{s}}					
\safemath{\tfcf}{r_{l}}						
\safemath{\bfcf}{r_{H}}						

\safemath{\tcorr}{c_h}						
\safemath{\scf}{c_{s}}						
\safemath{\tfcorr}{c_{l}}					
\safemath{\fcorr}{c_{H}}						

\safemath{\mi}{I}							
\safemath{\capacity}{C}						

\safemath{\normal}{\mathcal{N}}			
\safemath{\jpg}{\mathcal{CN}}			
\safemath{\mchain}{\leftrightarrow}		

\safemath{\dB}{\,\mathrm{dB}}
\safemath{\dBm}{\,\mathrm{dBm}}
\safemath{\Hz}{\,\mathrm{Hz}}
\safemath{\kHz}{\,\mathrm{kHz}}
\safemath{\MHz}{\,\mathrm{MHz}}
\safemath{\GHz}{\,\mathrm{GHz}}
\safemath{\s}{\,\mathrm{s}}
\safemath{\ms}{\,\mathrm{ms}}
\safemath{\mus}{\,\mathrm{\text{\textmu}s}}
\safemath{\ns}{\,\mathrm{ns}}
\safemath{\ps}{\,\mathrm{ps}}
\safemath{\meter}{\,\mathrm{m}}
\safemath{\mm}{\,\mathrm{mm}}
\safemath{\cm}{\,\mathrm{cm}}
\safemath{\m}{\,\mathrm{m}}
\safemath{\W}{\,\mathrm{W}}
\safemath{\mW}{\, \mathrm{mW}}
\safemath{\J}{\,\mathrm{J}}
\safemath{\K}{\,\mathrm{K}}
\safemath{\bit}{\,\mathrm{bit}}
\safemath{\nat}{\,\mathrm{nat}}

\safemath{\define}{\triangleq}			

\safemath{\equivalent}{\sim}
\safemath{\distas}{\sim}					
\safemath{\sdiff}{\Delta}				

\safemath{\reals}{\mathbb{R}}
\safemath{\positivereals}{\reals_{+}}
\safemath{\integers}{\mathbb{Z}}
\safemath{\posint}{\integers_{+}}
\safemath{\naturals}{\mathbb{N}}
\safemath{\posnaturals}{\naturals_{+}}
\safemath{\complexset}{\mathbb{C}}
\safemath{\rationals}{\mathbb{Q}}

\newcommand*{\fancyrefapplabelprefix}{app}		
\newcommand*{\fancyrefthmlabelprefix}{thm}		
\newcommand*{\fancyreflemlabelprefix}{lem}		
\newcommand*{\fancyrefcorlabelprefix}{cor}		
\newcommand*{\fancyrefdeflabelprefix}{def}		
\newcommand*{\fancyrefproplabelprefix}{prop}		
\newcommand*{\fancyrefobslabelprefix}{obs}		
\newcommand*{\fancyrefexmpllabelprefix}{exmpl}
\newcommand*{\fancyrefalglabelprefix}{alg}		

\frefformat{vario}{\fancyrefseclabelprefix}{Section~#1}
\frefformat{vario}{\fancyrefthmlabelprefix}{Theorem~#1}
\frefformat{vario}{\fancyreflemlabelprefix}{Lemma~#1}
\frefformat{vario}{\fancyrefcorlabelprefix}{Corollary~#1}
\frefformat{vario}{\fancyrefdeflabelprefix}{Definition~#1}
\frefformat{vario}{\fancyrefobslabelprefix}{Observation~#1}
\frefformat{vario}{\fancyreffiglabelprefix}{Fig.~#1}
\frefformat{vario}{\fancyrefapplabelprefix}{Appendix~#1} 
\frefformat{vario}{\fancyrefeqlabelprefix}{(#1)}
\frefformat{vario}{\fancyrefproplabelprefix}{Proposition~#1}
\frefformat{vario}{\fancyrefexmpllabelprefix}{Example~#1}
\frefformat{vario}{\fancyrefalglabelprefix}{Algorithm~#1}

\newtheorem{thm}{Theorem}
\newtheorem{cor}[thm]{Corollary}   

\newtheorem{lem}[thm]{Lemma} 

\newtheorem{rem}{Remark}

\safemath{\dictab}{[\,\dicta\,\,\dictb\,]}

\safemath{\ysig}{\bmy}
\safemath{\ysighat}{\hat{\ysig}}
\safemath{\ysigdim}{M}
\safemath{\xsig}{\bmx}
\safemath{\xsigdim}{N}
\safemath{\nx}{n_x}
\safemath{\zsig}{\bmz}
\safemath{\zsigdim}{\ysigdim}
\safemath{\rsig}{\bmr}
\safemath{\Adict}{\bA}
\safemath{\Adicttilde}{\widetilde{\Adict}}
\safemath{\Adictdim}{\outputdim\times\xsigdim}
\safemath{\avec}{\bma}
\safemath{\avectilde}{\tilde{\avec}}
\safemath{\Bdict}{\bB}
\safemath{\Bdicttilde}{\widetilde{\Bdict}}
\safemath{\Cdict}{\bC}
\safemath{\cvec}{\bmc}
\safemath{\Ddict}{\bD}
\safemath{\Ddictdim}{\ysigdim\times\xsigdim}
\safemath{\dvec}{\bmd}
\safemath{\Ddicttilde}{\widetilde{\bD}}
\safemath{\Bonb}{\bB}
\safemath{\bvec}{\bmb}
\safemath{\Bonbdim}{\ysigdim\times\ysigdim}
\safemath{\noise}{\bmn}
\safemath{\noisedim}{\ysigim}
\safemath{\err}{\bme}
\safemath{\errdim}{\ysigdim}
\safemath{\errset}{\setE}
\safemath{\nerr}{n_e}
\safemath{\delop}{\bP_\errset}
\safemath{\delopc}{\bP_{{\errset}^c}}

\safemath{\cplxi}{\imath}
\safemath{\cplxj}{\jmath}

\safemath{\dict}{\matD}
\safemath{\inputdim}{N}		
\safemath{\outputdim}{M}		
\safemath{\sparsity}{S}	
\safemath{\inputdimA}{{N_a}}	
\safemath{\inputdimB}{{N_b}}	
\safemath{\elemA}{{n_a}}	
\safemath{\elemB}{{n_b}}	
\safemath{\resA}{\matR_a}	
\safemath{\resB}{\matR_b}	
\safemath{\subD}{\matS} 
\safemath{\subA}{\matS_a} 
\safemath{\subB}{\matS_b} 
\safemath{\dicta}{\matA} 	
\safemath{\dictb}{\matB} 	
\safemath{\hollowS}{H}
\safemath{\hollowA}{H_a}
\safemath{\hollowB}{H_b}
\safemath{\cross}{Z}
\safemath{\coh}{\mu_d}			
\safemath{\coha}{\mu_a}			
\safemath{\cohb}{\mu_b}			
\safemath{\mubs}{\nu}	
\safemath{\cohm}{\mu_m} 
\safemath{\dictset}{\setD}	
\safemath{\dictsetp}{\dictset(\coh,\coha,\cohb)}	
\safemath{\dictsetgen}{\dictset_\text{gen}}
\safemath{\dictsetgenp}{\dictsetgen(\coh)}
\safemath{\dictsetonb}{\dictset_\text{onb}}
\safemath{\dictsetonbp}{\dictsetonb(\coh)}

\safemath{\leftside}{U}
\safemath{\rightsideA}{R_a}
\safemath{\rightsideB}{R_b}

\safemath{\indexS}{\setI_S} 

\safemath{\na}{n_a}			
\safemath{\nb}{n_b}			
\safemath{\coeffa}{p_i}	
\safemath{\coeffb}{q_j}	
\safemath{\seta}{\setP}		
\safemath{\setb}{\setQ}     
\safemath{\setw}{\setW}	
\safemath{\setz}{\setZ}	
\safemath{\cola}{\veca}		
\safemath{\colb}{\vecb}		
\safemath{\cold}{\vecd}		
\safemath{\inputvec}{\vecx} 	
\safemath{\error}{\vece}	
\safemath{\noiseout}{\vecz} 	
\safemath{\inputvecel}{x}
\safemath{\inputveca}{\vecx_a}
\safemath{\inputvecb}{\vecx_b}
\safemath{\outputvec}{\vecy}	
\safemath{\lambdamin}{\lambda_{\mathrm{min}}}

\safemath{\elltwo}{\ell_2}
\safemath{\ellone}{\ell_1}
\safemath{\ellzero}{\ell_0}
\safemath{\ellinf}{\ell_\infty}
\safemath{\ellinftilde}{\ell_{\widetilde\infty}}
\safemath{\licard}{Z(\coh,\coha,\cohb)}
\safemath{\xsol}{\hat{x}}
\safemath{\xbord}{x_b}		
\safemath{\xstat}{x_s}		
\safemath{\xstatLone}{\tilde{x}_s}
\safemath{\order}{\mathcal{O}} 
\safemath{\scales}{\Theta} 
\safemath{\ones}{\mathbf{1}} 
\safemath{\zeroes}{\mathbf{0}} 
\safemath{\thlone}{\kappa(\coh,\cohb)} 
\safemath{\constoneA}{\delta} 
\safemath{\constoneB}{\epsilon} 
\safemath{\nlarge}{L}				   
\safemath{\sumlarge}{S_\nlarge}
\safemath{\maxlarger}{P_\nlarge}	   
\safemath{\Pzero}{\textrm{P0}}	
\safemath{\Pone}{\textrm{P1}}
\safemath{\vecfir}{\vecw}			 
\safemath{\vecsec}{\vecz}
\safemath{\elvecfir}{w}              
\safemath{\elvecsec}{z}				 
\safemath{\nlargefir}{n}
\safemath{\normout}{\gamma}
\safemath{\auxfun}{h}
\safemath{\supp}{\textrm{supp}}

\safemath{\indexa}{\ell}
\safemath{\indexb}{r}
\safemath{\indexc}{i}
\safemath{\indexd}{j}

\safemath{\project}{P}

\allowdisplaybreaks
\IEEEoverridecommandlockouts

\usepackage{color}

\begin{document}

\title{Linearized Binary Regression
}

\author{
\IEEEauthorblockN{Andrew S. Lan$^{1}$, Mung Chiang$^{2}$, and Christoph Studer$^{3}$}\\ 
\IEEEauthorblockA{$^\text{1}$Princeton University, Princeton, NJ; andrew.lan@princeton.edu} 
\IEEEauthorblockA{$^\text{2}$Purdue University, West Lafayette, IN; chiang@purdue.edu} 
\IEEEauthorblockA{$^\text{3}$Cornell University, Ithaca, NY; studer@cornell.edu}
\thanks{AL and MC were supported in part by the US National Science Foundation (NSF) under grant CNS-1347234. CS was supported in part by Xilinx Inc.~and by the US NSF under grants ECCS-1408006, CCF-1535897, CAREER CCF-1652065, and CNS-1717559.}
}

\maketitle



\begin{abstract}
Probit regression was first proposed by Bliss in 1934 to study mortality rates of insects. Since then, an extensive body of work has analyzed and used probit or related binary regression methods (such as logistic regression) in numerous applications and fields.
This paper provides a fresh angle to such well-established binary regression methods.
Concretely, we demonstrate that \emph{linearizing} the probit model in combination with \emph{linear estimators} performs \emph{on par} with state-of-the-art nonlinear regression methods, such as posterior mean or maximum a-posteriori estimation, for a broad range of real-world regression problems.
We derive \emph{exact}, \emph{closed-form}, and \emph{nonasymptotic} expressions for the mean-squared error of our linearized estimators, which clearly separates them from nonlinear regression methods that are typically difficult to analyze.
%
We showcase the efficacy of our methods and results for a number of synthetic and real-world datasets, which demonstrates that \emph{linearized binary regression} finds potential use in a variety of inference, estimation, signal processing, and machine learning applications that deal with binary-valued observations or measurements. 
\end{abstract}





\section{Introduction}
This paper deals with the estimation of the $N$-dimensional vector $\bmx\in\reals^N$ from the following measurement model:
\begin{align} \label{eq:probitmodel}
\bmy = \sign(\bD\bmx+\bmw).
\end{align}
Here, the vector $\bmy\in\{-1,+1\}^M$ contains $M$ binary-valued measurements, the  function $\sign(z)$ operates element-wise on its argument and outputs $+1$ for $z\geq0$ and $-1$ otherwise, $\bD\in\reals^{M\times N}$ is a given design matrix (or matrix of covariates).  The noise vector $\bmw\in\reals^M$ has i.i.d.\ random entries.  
Estimation of the vector $\bmx$ from the observation model in \fref{eq:probitmodel} is known as binary regression.  The two most common types of binary regression are (i) \emph{probit regression} \cite{firstprobit} for which the noise vector $\bmw$ follows a standard normal distribution and (ii) \emph{logistic regression}~\cite{logregfirst} for which the noise vector $\bmw$ follows a logistic distribution with unit scale parameter.  

Binary regression finds widespread use in a broad range of applications and fields, including (but not limited to) image classification~\cite{imageclassification}, biomedical data analysis~\cite{cancer,biology}, economics~\cite{economy}, and signal processing~\cite{larry,plan1bit}.  In most real-world applications, one can use either probit or logistic regression, since the noise distribution is unknown; in this paper, we focus on probit regression for reasons that we will detail in \fref{sec:lmmse}.  In what follows, we will assume that the noise vector $\bmw\in\reals^M$ has i.i.d.\ standard normal entries, and refer to~\fref{eq:probitmodel} as the standard probit model.  

\subsection{Relevant Prior Art}
\subsubsection{Estimators}
The two most common estimation techniques for the standard probit model in~\fref{eq:probitmodel} are the posterior mean (PM) and maximum a-posteriori (MAP) estimators. 
The PM estimator computes the following conditional expectation \cite{poorbook}:
\begin{align} \label{eq:PME}
\hat\bmx^\text{PM} =\textstyle  \Ex{\bmx}{\bmx| \bmy} = \int_{\reals^N} \bmx p(\bmx|\bmy)\text{d}\bmx,
\end{align}
where $p(\bmx|\bmy)$ is the posterior probability of the vector $\bmx$ given the observations $\bmy$ under the model  \fref{eq:probitmodel}. The PM estimator is optimal in a sense that it minimizes the mean-squared error (MSE) defined as
\begin{align} \label{eq:MSE}
\textit{MSE}(\hat\bmx) = \Ex{\bmx,\bmw}{\|\bmx-\hat\bmx\|^2},
\end{align} 
and is, hence, also known as the nonlinear minimum mean-squared error (MMSE) estimator. Evaluating the integral in~\fref{eq:PME} for the probit model is difficult and hence, one typically resorts to rather slow Monte-Carlo methods \cite{albertchib}. 
By assuming that the vector $\bmx$ is multivariate Gaussian, an alternative regression technique is the MAP estimator that solves the following convex optimization problem~\cite{probitml}:
\begin{align} \label{eq:MAPestimator}
\hat\bmx^\text{MAP} \!=\! \textstyle \underset{\bmx\in\reals^N}{\text{arg\,min}} \!  -\!\sum_{m=1}^M \log(\Phi(y_m \bmd_m^T \bmx)) + \frac{1}{2} \bmx^T \bC_{\bmx}^{-1} \bmx.
\end{align}
Here, $\Phi(x)=\int_{-\infty}^x (2\pi)^{-1/2}e^{-t^2/2}\text{d}t$ is the cumulative distribution function of a standard normal random variable, $\bmd^T_m$ is the $m$th row of the covariate matrix $\bD$, and $\bC_\bmx$ is the covariance matrix of the zero-mean multivariate Gaussian prior on the vector~$\bmx$. By ignoring the prior on $\bmx$, one arrives at the well-known maximum-likelihood (ML) estimator. 
Compared to the PM estimator, MAP and ML estimation can be implemented efficiently either by solving a series of re-weighted least squares problems \cite{tibsbook} or by using standard numerical methods for convex problems that scale favorably to large problem sizes~\cite{wrightbook,goldstein2014field}. 
In contrast to such well-established nonlinear estimators, we will investigate \emph{linear estimators} that are computationally efficient and whose performance is on par to that of the PM, MAP, and ML estimators. 

\subsubsection{Analytical Results}
Analytical results that characterize the performance of estimation under the probit model are almost exclusively for the \emph{asymptotic} setting, i.e., when $M$ and/or~$N$ tend to infinity.  
More specifically, Brillinger \cite{brillinger1982generalized} has shown in 1982 that the conventional least-squares (LS) estimators for scenarios in which the design matrix $\bD$ has i.i.d.\ Gaussian entries, delivers an estimate that is the same as that of the PM estimator up to a constant. 
More recently, Brillinger's result has been generalized by Thrampoulidis \emph{et al.}~\cite{thrampoulidis2015lasso} to the sparse setting, i.e., where the vector~$\bmx$ has only a few nonzero entries. 
%
%
Other related results analyze the consistency of the ML estimator for sparse logistic regression.  
These results are either asymptotic \cite{honestlogit,plan1bit,wainwrightising} or of probabilistic nature~ \cite{bachlogit}; the latter type of results bounds the MSE with high probability. 
%
In contrast to all such existing analytical results, we will provide \emph{nonasymptotic} and \emph{exact} expressions for the MSE that are valid for \emph{arbitrary} and \emph{deterministic} design matrices $\bD$.

\subsection{Contributions}
We propose novel \emph{linear estimators} of the form $\hat\bmx=\bW\bmy$ for the probit model in \fref{eq:probitmodel}, where $\bW\in\reals^{N\times M}$ are suitably-chosen estimation matrices, and  provide \emph{exact}, \emph{closed-form}, and \emph{nonasymptotic} expressions for the MSE of these estimators.  
Specifically, we will develop two estimators: a linear minimum mean-squared error (L-MMSE) estimator that aims at minimizing the MSE in \fref{eq:MSE}  and a more efficient but less accurate least-squares~(LS) estimator.
Our MSE results are in stark contrast to existing performance guarantees for the MAP or PM estimators, for which a nonasymptotic performance analysis is, in general, difficult.  
%
%
We provide inference results on synthetic data, which suggest that the inference quality of the proposed linear estimators is on par with state-of-the-art nonlinear estimators, especially at low signal-to-noise ratio~(SNR), i.e., when the quantization error  is lower than the noise level.  Moreover, we show using six different real-world binary regression datasets that the proposed linear estimators achieve competitive predictive performance to PM and MAP estimation at comparable or even lower complexity.



\section{Linearized Probit Regression}
\label{sec:mainresults}

To develop and analyze linearized inference methods for the standard probit model in \fref{eq:probitmodel}, we will first consider the following smoothed version of the probit model:
\begin{align} \label{eq:probitmodelgeneral}
\bar{\bmy} = f_\sigma(\bD\bmx+\bmw).
\end{align} 
We will then use these results to study the binary model \fref{eq:probitmodel}.
Here, $\bar{\bmy}\in[-1,+1]^M$, $\bmx$ is zero-mean Gaussian with known covariance~$\bC_\bmx$, the sigmoid function is defined as 
$f_\sigma(z) = 2 \Phi(z/\sigma) - 1$ and operates element-wise on its argument, $\sigma\in(0,\infty)$ is a smoothing parameter, and the vector $\bmw$ is assumed to be zero-mean Gaussian with known covariance~$\bC_\bmw$ and independent of $\bmx$.\footnote{We emphasize that these are standard model assumptions in Bayesian data analysis (see, e.g., \cite{hoffbook}) and in numerous real-world applications, such as modeling user responses to test items \cite{lordirt}.}
We emphasize that as $\sigma\to0$, the sigmoid function $f_\sigma(z)$ corresponds to the sign function and hence, the model in \fref{eq:probitmodelgeneral} includes the probit model in \fref{eq:probitmodel} as a special case. 
%
In what follows, we assume nondegenerate covariance matrices for $\bmx$ and $\bmw$, i.e., we assume that~$\bC_{\bmx}$ and $\bC_\bmw$ are both invertible. 
We next introduce two new linear estimators for this model and then, provide exact, closed-form, and nonasymptotic expressions for the associated MSEs.

\subsection{Linear Minimum Mean-Squared Error Estimator}
\label{sec:lmmse}

Our main  result is as follows.
\begin{thm} \label{thm:lmmseestimator}
The linear minimum mean-squared error (L-MMSE) estimate for the generalized probit model in \fref{eq:probitmodelgeneral} is
\begin{align} \label{eq:LMMSE}
\hat\bmx^\text{\em L-MMSE} = \bE^T \bC_{\bar{\bmy}}^{-1}\bar{\bmy},
\end{align}
where
\begin{align} 
\bE & =  \textstyle \left(\frac{2}{\pi}\right)^{1/2} \mathrm{diag}(\mathrm{diag}(\sigma^2\bI+\bC_\bmz)^{-1/2})\bD\bC_\bmx,\label{eq:Ematrixcompact} \\
\notag \bC_{\bar{\bmy}} & =\textstyle  \frac{2}{\pi} \arcsin ( \mathrm{diag}( \mathrm{diag}(\sigma^2\bI+\bC_\bmz)^{-1/2})  \bC_\bmz \\ & \quad \quad \times \mathrm{diag}(\mathrm{diag}(\sigma^2\bI+\bC_\bmz)^{-1/2}) ), \label{eq:arcsinelawcompact}
\end{align}
and $\bC_\bmz=\bD\bC_\bmx\bD^T+\bC_\bmw$.
\end{thm}

\begin{rem}
The reason that we focus on probit regression is that under the standard probit model, the matrices $\bE$ and $\bC_{\bar{\bmy}}$ exhibit closed-form expressions;  For logistic regression, such closed-form expressions do not exist. 
\end{rem}


\begin{proof}
The proof consists of two steps. First, we linearize the model in \fref{eq:probitmodelgeneral}.
Then, we derive the L-MMSE estimate in~\fref{eq:LMMSE} for the linearized model. The two  steps are as follows. 

{\em Step 1 (Linearization): } Let $\bmz=\bD\bmx+\bmw$ and 
\begin{align} \label{eq:bussgang}
\bar{\bmy}=f_\sigma(\bmz)=\bF\bmx+\bme
\end{align} be a linearization of the generalized probit model in \fref{eq:probitmodelgeneral}, where $\bF\in\reals^{M\times N}$  is a \emph{linearization matrix} and $\bme\in\reals^M$ is a \emph{residual error vector} that contains noise and linearization artifacts. 
Our goal is to perform a Bussgang-like decomposition \cite{bussgang}, which uses the linearization matrix~$\bF$ that minimizes the $\ell_2$-norm of the residual error vector~$\bme$ averaged over the signal and noise.
Concretely, let~$\bC_\bmz$ be the covariance matrix of the vector $\bmz$ and consider the optimization problem
\begin{align*}
\underset{\bF\in\reals^{M\times N}}{\mathrm{minimize}} \; \Ex{\bmx,\bmw}{\big\|\bar{\bmy} - \bF\bmx\big\|^2}\!,
\end{align*}
which has a closed-form solution that is given by $\bF = \bE \bC^{-1}_\bmx$ with $\bE= \Ex{\bmx,\bmw}{\bar{\bmy}\bmx^T}$.
It can easily be verified that for this particular choice of the linearization matrix $\bF$, the residual error vector $\bme$ and the signal of interest $\bmx$ are uncorrelated, i.e., we have $\Ex{\bmx,\bmw}{\bmx\bme^T}=\bZero_{N\times M}$.

We now derive a closed-form expression for the entries of the matrix $\bE$.  Since both $\bmx$ and $\bmw$ are independent and zero-mean Gaussian, the bivariate $(z_m,x_n)$ is jointly Gaussian for each index pair $\{m,n\}$. Moreover, we have $\Ex{t}{\abs{f_\sigma(t)}} < \infty$ and $\Ex{t}{\abs{t f_\sigma(t)}} < \infty$ if~$t$ is a zero-mean Gaussian random variable. Hence, we can use the following result that is due to Brillinger \cite[Lem.~1]{brillinger}:
\begin{align} \label{eq:brillinger}
[\bE]_{m,n}  = \Ex{\bmx,\bmw}{\bar{y}_mx_n} = \frac{\textit{Cov}(z_m,x_n)}{\textit{Var}(z_m)} \Ex{z_m}{\bar{y}_mz_m}\!,
\end{align}
where $\textit{Cov}(z_m,x_n)= \bmd_m^T\bmc_n$ with $\bmc_n$ being the $n$th column of $\bC_\bmx$.
Since for $\sigma>0$ the function $\bar{y}_m= f_\sigma(z_m)$ is absolutely continuous\footnote{The special case for $f_0(z_m)$ can either be derived by directly evaluating $\Ex{}{\sign(z_m)z_m}$ in \fref{eq:brillinger} or by first using Stein's Lemma and then letting $\sigma\to0$; both approaches yield the same result.}, $z_m$ is zero-mean Gaussian, and $\Ex{t}{f'_\sigma(t)} < \infty$,  we can invoke Stein's Lemma~\cite{stein}, which states that
\begin{align} \label{eq:stein}
\frac{\Ex{z_m}{f_\sigma(z_m)z_m}}{\textit{Var}(z_m)} = \Ex{z_m}{f'_\sigma(z_m)},
\end{align}
with $f'_\sigma(z)=\frac{\text{d}}{\text{d}z} f_\sigma(z)$.
Using $f_\sigma(x) = 2 \Phi(x/\sigma) - 1$, we can evaluate the right-hand side in \fref{eq:stein} as 
\begin{align}
 \Ex{z_m}{f'_\sigma(z_m)} & \textstyle = 2 \Ex{z_m}{ \Phi'\!\left({z_m}/{\sigma}\right)} \notag \\
& = \textstyle \frac{2}{\sigma} \int_{-\infty}^\infty \mathcal{N} \!\left({z_m}/{\sigma}; 0,1\right) \mathcal{N} (z_m; 0, \gamma_m) \mathrm{d} z_m  \notag \\
& \textstyle = \frac{2}{\sigma} \frac{\sigma'}{\sqrt{2 \pi\gamma_m}}  \int_{-\infty}^\infty \frac{1}{\sqrt{2 \pi} \sigma'} \exp\!\left(-\frac{z_m^2}{2 \sigma'^2}\right)\! \mathrm{d} z_m \notag \\
& = \textstyle \left(\frac{2}{\pi}\right)^{1/2}\! \frac{1}{\sqrt{\sigma^2 + \gamma_m}}, \label{eq:laststep1}
\end{align}
where  $\mathcal{N} (z; \mu, \sigma^2)$ denotes the probability density function of a Gaussian distribution with mean $\mu$ and variance $\sigma^2$ evaluated at $z$, $\gamma_m=\textit{Var}(z_m) = \bmd_m^T \bC_\bmx \bmd_m + [\bC_\bmw]_{m,m}$, and $\sigma'^2 = \frac{\sigma^2 \gamma_m}{\sigma^2 + \gamma_m}$.
Combining \fref{eq:brillinger} with \fref{eq:stein} and \fref{eq:laststep1} leads to 
\begin{align*}
[\bE]_{m,n} \textstyle = \left(\frac{2}{\pi}\right)^{1/2}\! \frac{ \bmd_m^T\bmc_n}{\sqrt{\sigma^2 + \bmd_m^T \bC_\bmx \bmd_m + [\bC_\bmw]_{m,m}}},
\end{align*}
where \fref{eq:Ematrixcompact} represents the entire matrix $\bE$ in compact notation.

%

{\em Step 2 (L-MMSE Estimator): }
We have linearized the probit model as $\bar{\bmy}=f_\sigma(\bmz)=\bF\bmx+\bme$ in~\fref{eq:bussgang} with $\bF = \bE \bC^{-1}_\bmx$. We now estimate $\bmx$ from this linearization using the L-MMSE estimator.
Since the residual distortion vector $\bme$ is uncorrelated to the vector $\bmx$, the L-MMSE estimator is given by  
\begin{align*}
\hat\bmx^\text{L-MMSE} = \bE^T \bC_{\bar{\bmy}}^{-1}\bar{\bmy},
\end{align*}
where $\bC_{\bar{\bmy}}=\Ex{\bmx,\bmw}{\bar{\bmy}\bar{\bmy}^T}$ is the covariance matrix of the generalized probit measurements in \fref{eq:probitmodelgeneral}.
The remaining piece is to calculate the individual entries of this matrix. 

With abuse of notation, we start by deriving the necessary expressions for a general pair of correlated but zero-mean Gaussian random variables $(x,y)$ with covariance matrix $\bC = [C_{x,x},C_{x,y};  C_{x,y} ,C_{y,y}]$.
More specifically, we are interested in computing the quantity 
\begin{align*}
\Ex{x,y}{f_\sigma(x)f_\sigma(y)} 
& = 4 \Ex{x,y}{\Phi({x}/{\sigma}) \Phi({y}/{\sigma})} + 1 \\
& \quad - 2\Ex{x}{\Phi({x}/{\sigma})} - 2\Ex{y}{\Phi({y}/{\sigma})} .
\end{align*}
Since
\begin{align*}
& \Ex{x}{\Phi({x}/{\sigma})} = \textstyle \int_{-\infty}^\infty \Phi({x}/{\sigma}) \mathcal{N}(x; 0, C_{x,x}) \mathrm{d} x  \\
 & \quad = \textstyle \int_0^\infty (\Phi(-{x}/{\sigma}) + \Phi({x}/{\sigma}))  \mathcal{N}(x; 0, C_{x,x}) \mathrm{d} x 
 = \frac{1}{2},
\end{align*}
we have
\begin{align} \label{eq:momentsrelation}
\Ex{x,y}{f_\sigma(x)f_\sigma(y)} &= 4 \Ex{x,y}{\Phi({x}/{\sigma}) \Phi({y}/{\sigma})} -1.
\end{align}
Hence, we only need a closed-form expression for $\Ex{x,y}{\Phi({x}/{\sigma}) \Phi({y}/{\sigma})}$, which we derive using direct integration. 
We rewrite this expression as follows:
\begin{align*}
& \Ex{x,y}{\Phi\!\left(\frac{x}{\sigma}\right)\! \Phi\!\left(\frac{y}{\sigma}\right)} 
= \int_{-\infty}^\infty \int_{-\infty}^\infty \Phi\!\left(\frac{x}{\sigma}\right) \Phi\!\left(\frac{y}{\sigma}\right) \\ 
& \quad \times \mathcal{N} \Big( \Big[ \begin{array}{c} x \\ y \end{array} \Big]; \bm0, \Big[ \begin{array}{cc} C_{x,x} & C_{x,y} \\ C_{x,y} & C_{y,y} \end{array} \Big] \Big) \mathrm{d}x \mathrm{d}y \\
& \quad = \int_{-\infty}^\infty \Phi\!\left(\frac{x}{\sigma}\right)\! \Phi\!\left(\frac{x}{\sigma'}\right) \mathcal{N}(x; 0, C_{x,x}) \mathrm{d}x,
\end{align*}
where the last equality follows from \cite[Sec.~3.9]{gpml} with $\sigma' = \frac{C_{x,x}}{C_{x,y}} \sqrt{\sigma^2 + C_{y,y} + \frac{C_{x,y}^2}{C_{x,x}}}$.  We now further simplify the above expression with the following steps:
\begin{align*}
& \Ex{x,y}{\Phi\!\left(\frac{x}{\sigma}\right)\! \Phi\!\left(\frac{y}{\sigma}\right)} \\
& \quad = \int_{-\infty}^\infty \Phi\!\left(\frac{x}{\sigma}\right)\! \Phi\!\left(\frac{x}{\sigma'}\right)\! \frac{1}{\sqrt{2 \pi C_{x,x}}} \exp\!\left(-\frac{x^2}{2 C_{x,x}}\right)\! \mathrm{d}x \\
&  \quad =  \int_{-\infty}^\infty \Phi\bigg(\frac{\sqrt{C_{x,x}}}{\sigma}x\bigg) \Phi\bigg(\frac{\sqrt{C_{x,x}}}{\sigma'}x\bigg) \mathcal{N}(x; 0,1) \mathrm{d} x.
\end{align*}
Using the definitions $\sigma_1 = {\sigma}/\!{\sqrt{C_{x,x}}}$ and $\sigma_2= {\sigma'}/\!{\sqrt{C_{x,x}}}$, we can rewrite the above expression as
\begin{align*}
& \Ex{x,y}{\Phi\!\left(\frac{x}{\sigma}\right)\! \Phi\!\left(\frac{y}{\sigma}\right)} = \int_{-\infty}^\infty \int_{-\infty}^{\frac{x}{\sigma_1}} \int_{-\infty}^{\frac{x}{\sigma_2}} \\
& \quad \quad \qquad \mathcal{N}(y; 0,1)  \mathcal{N}(z; 0,1) \mathrm{d} z \mathrm{d} y\, \mathcal{N}(x; 0,1) \mathrm{d} x.
\end{align*}
To evaluate this expression, it is key to observe that it corresponds to the cumulative probability density of a 3-dimensional normal random variable with zero mean and an identity covariance matrix on a region cut by two planes.  Imagine a cuboid with edge lengths $\{1, {1}/{\sigma_1}, {1}/{\sigma_2}\}$.  Assume $C_{x,y} > 0$ without loss of generality. The first plane has the normal vector $[1, -\sigma_1, 0]^T$, while the second plane has the normal vector $[1, 0, -\sigma_2]^T$. 
To find a convenient way to evaluate this integral, we need to find an appropriate change of coordinates.  Define the first new coordinate $x'$ as the intersection of the two planes, along the direction of $[1, {1}/{\sigma_1}, {1}/{\sigma_2}]^T$.  With proper normalization, this implies $x' = \frac{\sigma_1 \sigma_2 x + \sigma_2 y + \sigma_1 z}{\sqrt{\sigma_1^2 \sigma_2^2 + \sigma_1^1 + \sigma_2^2}}$.  
Then, we let the second coordinate $y'$ be orthogonal to $x'$ and also to the first plane, i.e., orthogonal to the normal vector of the first plane, $[1, {1}/{\sigma_1}, {1}/{\sigma_2}]^T$.  This gives $y' = \frac{\sigma_1^2 x + \sigma_1 y - \sigma_2 (\sigma_1^2 + 1) z}{\sqrt{\sigma_1^2 + 1}\sqrt{\sigma_1^2 \sigma_2^2 + \sigma_1^1 + \sigma_2^2}}$. 
The third coordinate is simply $z' = \frac{x - \sigma_1 y}{\sqrt{\sigma_1^2 + 1}}$, taken as the normal vector to the first plane.  The unit vector in the second plane that is orthogonal to $x'$ and $y'$ is given by $v' = \frac{\sigma_2^2 x - \sigma_1(\sigma_2^2 + 1) y - \sigma_2 z}{\sqrt{\sigma_2^2 + 1}\sqrt{\sigma_1^2 \sigma_2^2 + \sigma_1^1 + \sigma_2^2}}$. 
Since the new coordinates form a Cartesian system and are properly normalized, the determinant of the Jacobian is one, and the covariance matrix of the 3-dimensional normal random variable remains an  identity matrix. We first integrate over~$x'$ to obtain
\begin{align*}
\textstyle \Ex{x,y}{\Phi\!\left(\frac{x}{\sigma}\right)\! \Phi\!\left(\frac{y}{\sigma}\right)} = & \textstyle \int\!\! \int_\mathcal{C} \mathcal{N}(y'; 0,1) \mathcal{N}(z'; 0,1) \mathrm{d} y' \mathrm{d} z' ,
\end{align*}
where we have used $\mathcal{C}$ to denote the space to integrate over for the variables $y'$ and $z'$.  Since $\mathcal{C}$ is the area between the directions of $y'$ and $v'$ in the 2-dimensional plane, we use polar coordinates $y' = \rho \cos\theta$ and $z' = \rho \sin\theta$ to get
\begin{align*} 
& \textstyle \Ex{x,y}{\Phi\!\left(\frac{x}{\sigma}\right)\! \Phi\!\left(\frac{y}{\sigma}\right)} \\
& \,\, =   \textstyle \int_0^{\frac{\pi}{2}+\arcsin\!\left( \frac{1}{\sqrt{\sigma_1^2 + 1} \sqrt{\sigma_2^2 + 1}}\right)} \int_0^\infty \frac{1}{2 \pi} e^{-\frac{\rho^2}{2}} \rho \mathrm{d}\rho \mathrm{d} \theta \\
& \,\, = \textstyle  \frac{1}{4} + \frac{1}{2 \pi} \arcsin\!\left(\frac{C_{x,y}}{\sqrt{\sigma^2 + C_{x,x}} \sqrt{\sigma^2 + C_{y,y}}}\right)\!. 
\end{align*}
Consequently, we have
\begin{align*} 
\Ex{x,y}{f_\sigma(x)f_\sigma(y)} =\textstyle  \frac{2}{\pi} \arcsin\!\left(\frac{C_{x,y}}{\sqrt{\sigma^2 + C_{x,x}} \sqrt{\sigma^2 + C_{y,y}}}\right)\!,
\end{align*}
which allows us, in combination with \fref{eq:momentsrelation},  to express the desired covariance matrix $\bC_{\bar{\bmy}}$ as in \fref{eq:arcsinelawcompact}.  
\end{proof}

For the L-MMSE estimator in \fref{thm:lmmseestimator}, we can extract the MSE in closed form:
\begin{lem} \label{lem:mseoflmmseestimator}
The MSE of the L-MMSE estimator in \fref{thm:lmmseestimator} is given by
\begin{align*}
\textit{MSE}(\hat\bmx^\text{\em L-MMSE}) = \tr(\bC_\bmx - \bE^T \bC_{\bar{\bmy}}^{-1} \bE).
\end{align*}
\end{lem}
\begin{proof}
The proof follows from the MSE definition  in \fref{eq:MSE} and the facts that $\bF = \bE \bC^{-1}_\bmx$ and the two vectors $\bmx$ and $\bme$ are uncorrelated for the L-MMSE estimator in \fref{eq:LMMSE}.
\end{proof}

By letting the parameter $\sigma\to0$ in \fref{eq:probitmodelgeneral}, we can use \fref{thm:lmmseestimator} and 
\fref{lem:mseoflmmseestimator} to obtain the following corollary for the standard probit model in \fref{eq:probitmodel}.
This result agrees with a recent result in wireless communications~\cite{mimo}.  
\begin{cor} \label{cor:lmmseestimator}
The L-MMSE estimate for the standard probit model in \fref{eq:probitmodel} is $\hat\bmx^\text{\em L-MMSE} = \bE^T \bC_\bmy^{-1}\bmy$,
where 
\begin{align*}
\bE & = \textstyle \left(\frac{2}{\pi}\right)^{1/2} \mathrm{diag}(\mathrm{diag}(\bC_\bmz)^{-1/2})\bD\bC_\bmx, \\
\bC_\bmy &= \textstyle \frac{2}{\pi} \arcsin ( \mathrm{diag}(\mathrm{diag}(\bC_\bmz)^{-1/2}) \bC_\bmz \\
& \quad \, \times \mathrm{diag}(\mathrm{diag}(\bC_\bmz)^{-1/2})),
\end{align*}
and $\bC_\bmz=\bD\bC_\bmx\bD^T+\bC_\bmw$. The associated MSE is given by $\textit{MSE}(\hat\bmx^\text{\em L-MMSE}) = \tr(\bC_\bmx - \bE^T \bC_\bmy^{-1} \bE)$.
\end{cor}

\subsection{Least Squares (LS) Estimator}
The L-MMSE estimator as in \fref{eq:LMMSE}  requires the computation of~$\bC_{\bar{\bmy}}$ followed by a matrix inversion. For large-scale problems, one can avoid the matrix inversion by first solving $\bmy=\bC_{\bar{\bmy}} \bmq$  for $\bmq$ using conjugate gradients \cite{wrightbook}, followed by calculating $\hat\bmx^\text{L-MMSE} = \bE^T\bmq$.  Hence, the complexity of L-MMSE estimation is comparable to that of MAP estimation. 
Computation of~$\bC_{\bar{\bmy}}$, however, cannot be avoided entirely. 

Fortunately, there exists a simpler linear estimator that avoids computation of $\bC_{\bar{\bmy}}$ altogether, which we call the least-squares~(LS) estimator.
Concretely, let $M\geq N$ and consider the linearization in  \fref{eq:bussgang}, which is 
$\bar{\bmy}=f_\sigma(\bmz)= \bE \bC^{-1}_\bmx\bmx+\bme$.
By ignoring the residual error vector $\bme$ and by assuming that the columns of $\bE$ are linearly independent, we can simply invert the matrix $\bE \bC^{-1}_\bmx$, which yields the LS estimate
\begin{align} \label{eq:LSestimate}
\hat\bmx^\text{LS} = \bC_\bmx \bE^+\bar{\bmy},
\end{align}
where $\bE^+=(\bE^T\bE)^{-1}\bE^T$ is the left pseudo-inverse of $\bE$. Again, one can use conjugate gradients to implement \fref{eq:LSestimate}.
In contrast to the L-MMSE estimator, the LS estimator does not require knowledge of $\bC_{\bar{\bmy}}$, which makes it more efficient yet slightly less accurate (see the experimental results section for a comparison). 
As for the L-MMSE estimator, we have a closed-form expression for the MSE of the LS estimator.  
\begin{lem}  \label{lem:mseoflsestimator}
Assume that $\bE^+$ exists. Then, the MSE of the LS estimator in \fref{eq:LSestimate} is given by
\begin{align*}
\textit{MSE}(\hat\bmx^\text{\em LS}) = \tr(\bC_\bmx \bE^+\bC_{\bar{\bmy}}(\bE^+)^T\bC_\bmx -\bC_\bmx).
\end{align*}
\end{lem}
\begin{proof}
The proof follows from the MSE definition \fref{eq:MSE}, and the facts that  $\Ex{}{\bar{\bmy}\bmx^T}=\bE$ and $\bE^+\bE=\bI$.
\end{proof}

\section{Numerical Results}
\label{sec:numericalresults}

We now experimentally demonstrate the efficacy of the proposed linear estimators.  
%
%

\subsection{Experiments with Synthetic Data}

We first compare the MSE of our estimators to that of the nonlinear MAP and PM estimators using synthetic data.  

%

\begin{figure*}[t]
\vspace{-0.5cm}
\centering
\subfigure[$M=10$, $N=5$.]{
\includegraphics[scale=0.33]{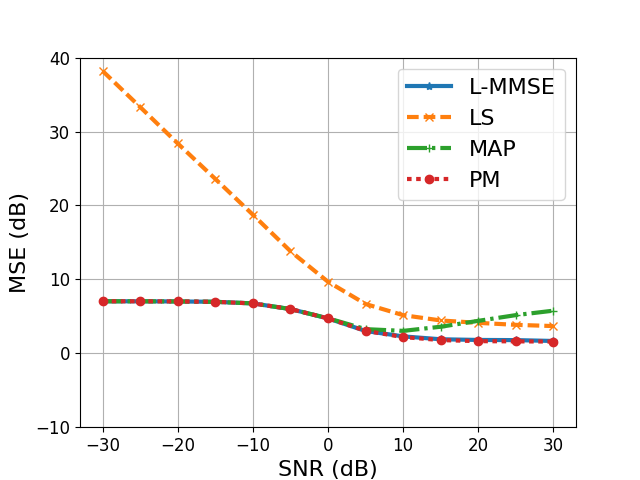}
} \hspace{0.4cm}
\subfigure[$M=50$, $N=5$.]{
\includegraphics[scale=0.33]{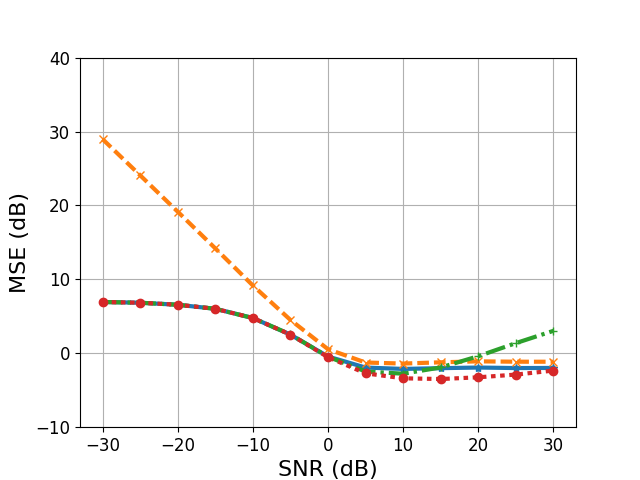}
} \hspace{0.4cm}
\subfigure[$M=200$, $N=5$.]{
\includegraphics[scale=0.33]{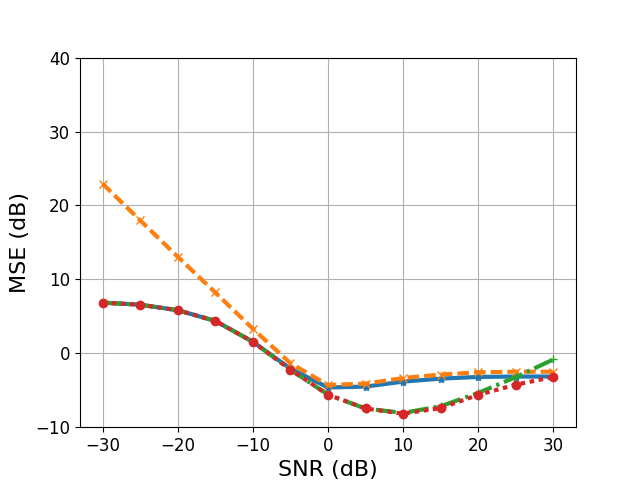}
}\\[-0.2cm]
\subfigure[$M=10$, $N=20$.]{
\includegraphics[scale=0.33]{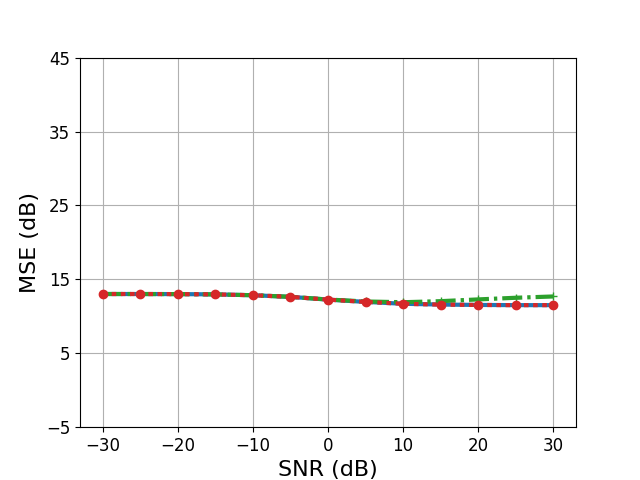}
} \hspace{0.4cm}
\subfigure[$M=50$, $N=20$.]{
\includegraphics[scale=0.33]{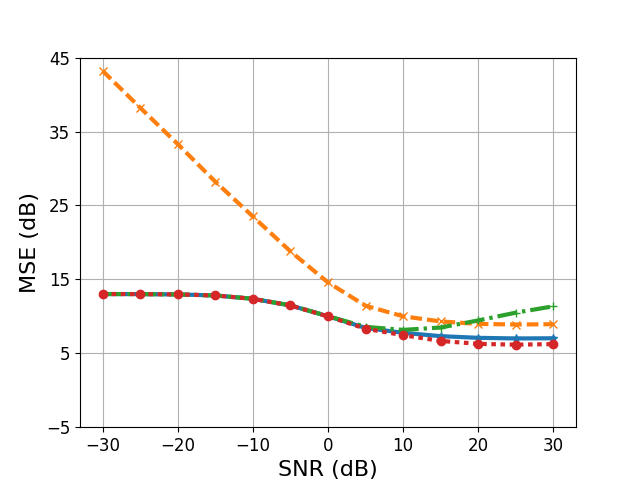}
} \hspace{0.4cm}
\subfigure[$M=200$, $N=20$.]{
\includegraphics[scale=0.33]{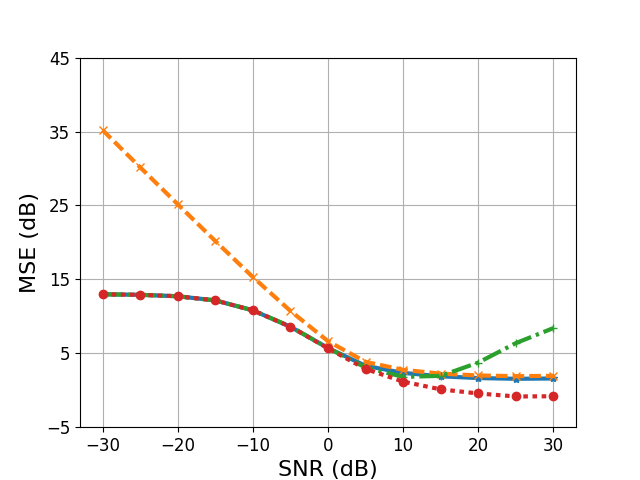}
}
\vspace{-0.1cm}
\caption{Mean squared error (MSE) versus signal-to-noise ratio (SNR) for various problem sizes $M \in \{ 10, 50, 200 \}$ and dimensions $N \in \{5, 20\}$. We see that for most parameter settings, the MSE of the proposed L-MMSE estimator is comparable to that of the optimal PM estimator; MAP estimation and LS estimation do not work as well at high and low SNR, respectively. }
\label{fig:compare}
\vspace{-0.0cm}
\end{figure*}


\subsubsection{Experimental Setup}
We set the dimensions of $\bmx$ to $N \in \{ 5, 20\}$ and the number of measurements to $M \in \{ 10, 50, 200\}$.  
We first generate a single random matrix $\bD$ of size $M \times N$ with i.i.d.\ standard normal entries, and normalize each row to have unit $\ell_2$-norm.  
Then, we generate the entries of $\bmx$ from a multivariate normal distribution with zero mean and covariance matrix $\bC_\bmx = \sigma_x^2 \bI$.  The entries of the noise vector $\bmw$ are i.i.d.\ zero-mean Gaussian with variance $\sigma_w^2$. The vector~$\bmy$ is generated using the standard probit model in \fref{eq:probitmodel}.
We sweep the SNR defined as $\textit{SNR} = {\sigma_x^2}/{\sigma_w^2}$ by changing the noise variance~$\sigma_w^2$.  
For the PM estimator, we use a standard Gibbs sampling procedure \cite{albertchib}; we use the mean of the generated samples over $50,000$ iterations as the PM estimate after a burn-in phase of $20,000$ iterations.
For the MAP estimator, we use an accelerated gradient-descent procedure~\cite{goldstein2014field,sparfa} to solve \fref{eq:MAPestimator} up to machine precision with a maximum number of $20,000$ iterations. 
We repeat all experiments for $100$ trials and report the empirical MSE.

\subsubsection{Results and Discussion}
\fref{fig:compare} shows the MSE of the L-MMSE, LS, MAP, and PM estimators.  We do not show LS for $M = 10$ and $N=20$ as it does not exist if $M < N$.  
We see that at low SNR ($\textit{SNR} \leq 0$\,dB), the L-MMSE, MAP, and PM estimators achieve a similar MSE.  
Hence, linearizing the probit model does not entail a noticeable performance degradation if the measurements are noisy. 
At higher SNR, the performance of the different estimators varies. For a small number of measurements, the performance of L-MMSE estimation is superior to MAP estimation. For a large number of measurements (e.g., $M = 200$), the MSE of L-MMSE estimation is slightly higher than that of MAP estimation for some SNR values. 
We note that the MSE performance of MAP degrades with increasing SNR, and we observe that there is an optimal SNR level for MAP estimation; this observation is in line with those reported in \cite{1bitmc} for 1-bit matrix completion using ML-type estimators.  Per design, PM estimation achieves the lowest MSE for all configurations, but is notoriously slow.
We conclude that linearized probit regression entails a negligible MSE performance loss compared to PM estimation, for a wide range of parameter settings.  

\begin{table*}[tp] 
\renewcommand{\arraystretch}{1.05}
\vspace{-0.0cm}
\centering
\caption{Mean and standard deviation of prediction quality in terms of prediction accuracy (ACC) for the L-MMSE, LS, MAP, PM, and Logit-MAP estimators on various real-world datasets.}  \label{tbl:acc}
\vspace{-0.1cm}
\scalebox{0.9}{ 
\begin{tabular}{@{}lccccc@{}}
\toprule
 & L-MMSE & LS & MAP & PM & Logit-MAP \\
\midrule
Admissions & $0.691 \pm 0.036 $ & $0.691 \pm 0.038 $  & $0.692 \pm 0.037 $ & $\bf 0.693\pm 0.036$ & $0.692 \pm 0.040 $ \\
Lowbwt & $0.703 \pm 0.070 $ & $0.707 \pm 0.071 $  & $\bf 0.715 \pm 0.064 $ & $0.713 \pm 0.067 $ & $0.712 \pm 0.070 $ \\
Polypharm & $0.779 \pm 0.015 $ & $0.777 \pm 0.016 $  & $\bf 0.780 \pm 0.015 $ & $\bf 0.780 \pm 0.015 $ & $\bf 0.780 \pm 0.015 $ \\
Myopia & $0.882 \pm 0.025 $ & $0.879 \pm 0.024 $  & $\bf 0.890 \pm 0.022 $ & $\bf 0.890 \pm 0.023 $ & $\bf 0.890 \pm 0.022 $ \\
Uis & $0.745 \pm 0.010 $ & $\bf 0.746 \pm 0.041 $  & $0.736 \pm 0.041 $ & $0.737 \pm 0.041 $ & $0.736 \pm 0.041 $ \\
SAheart & $0.727 \pm 0.042 $  & $0.726 \pm 0.044 $  & $0.728 \pm 0.042 $ & $\bf 0.730 \pm 0.042 $ & $0.729 \pm 0.042 $ \\
\bottomrule
\end{tabular}
}
\end{table*}

\begin{table*}[tp] 
\vspace{-0.0cm}
\renewcommand{\arraystretch}{1.05}
\centering
\caption{Prediction quality in terms of the area under the receiver operating characteristic curve (AUC) for the L-MMSE, LS, MAP, PM, and Logit-MAP estimators on various real-world datasets.} \label{tbl:auc}
\vspace{-0.1cm}
\scalebox{0.90}{ 
\begin{tabular}{@{}lccccc@{}}
\toprule
 & L-MMSE  & LS & MAP & PM & Logit-MAP \\
\midrule
Admissions & $\bf 0.675 \pm 0.056 $ & $0.672 \pm 0.054 $  & $0.674 \pm 0.056 $ & $0.674 \pm 0.056 $ & $\bf 0.675 \pm 0.056 $ \\
Lowbwt & $\bf 0.716 \pm 0.076 $ & $0.712 \pm 0.081 $ & $\bf 0.716 \pm 0.076 $ & $0.713 \pm 0.076 $ & $0.711 \pm 0.080 $ \\
Polypharm & $0.728 \pm 0.022 $ & $0.728 \pm 0.022 $  & $0.728 \pm 0.022 $ & $0.728 \pm 0.022 $ & $\bf 0.729 \pm 0.022 $ \\
Myopia & $0.864 \pm 0.038 $ & $0.862 \pm 0.040 $ & $\bf 0.873 \pm 0.036 $ & $\bf 0.873 \pm 0.036 $ & $\bf 0.873 \pm 0.035 $ \\
Uis & $0.632 \pm 0.052 $ & $0.632 \pm 0.051 $ & $\bf 0.634 \pm 0.052 $ & $\bf 0.634 \pm 0.051 $ & $0.633 \pm 0.052 $ \\
SAheart & $0.769 \pm 0.049 $  & $0.768 \pm 0.049 $ & $0.770 \pm 0.049 $ & $\bf 0.771 \pm 0.049 $ & $\bf 0.771 \pm 0.049 $ \\
\bottomrule
\end{tabular}
}
\end{table*}

\subsection{Experiments with Real-World Data}


We now validate the performance of the proposed linearized estimators using a variety of real-world datasets.  Since the noise model in real-world datasets is generally unknown, we also consider the  performance of MAP estimation using the logistic noise model (indicated by ``Logit-MAP'').

\subsubsection{Datasets}
We use a range of standard binary regression datasets in this experiment.  
These datasets include (i) ``Admissions'', which consists of binary-valued graduate school admission outcomes and features of the applicants, with $M = 400$ and $N=3$, (ii) ``Lowbwt'', which consists of low child birthweight indicators and features of their parents, with $M = 109$ and $N=10$, (iii) ``Polypharm'', which consists of whether an adult takes more than one type of prescription and their features, with $M =3,499$ and $N = 15$, (iv)  ``Myopia'', which consists of myopia test outcomes for adults and their personal features, with $M = 575$ and $N=11$, (v) ``Uis'', which consists of treatment outcomes for AIDS patients and their personal features, with $M = 618$ and $N=15$,  and (vi) ``SAheart'', which consists of whether a person has heart disease and their features, with $M=462$ and $N=9$.  
The first five datasets are taken from \cite{datasetbook} and the last one is from~\cite{tibsbook}. 

\subsubsection{Experimental Setup}
We evaluate the prediction quality of the L-MMSE, LS, MAP, PM, and Logit-MAP estimators using five-fold cross validation.  
We randomly divide the entire dataset into five nonoverlapping subsets, use four folds of the data as the training set and the other fold as the test set. We use the training set to estimate $\bmx$, and use it to predict the binary-valued outcomes on the test set.  For all experiments, we fix $\sigma_w^2 = 1$.  Since the variance $\sigma_x^2$ serves as a regularization parameter for $\bmx$, we select an optimal value of $\sigma_x^2$ using grid search on a separate validation set. 
To assess the performance of these estimators, we deploy the two most common metrics that characterize prediction quality: prediction accuracy (ACC) and area under the receiver operating characteristic curve (AUC)~\cite{accauc}.  Both metrics take values in $[0,1]$ and larger values indicate better prediction performance.  

\subsubsection{Results and Discussion}

Tables~\ref{tbl:acc} and~\ref{tbl:auc} show the mean and standard deviation of the performance of each estimator on both metrics across 20 random training/test partitions of the datasets.
We observe that the performance of L-MMSE, MAP, PM, and Logit-MAP are virtually indistinguishable on most datasets.  
%
LS estimation is, with a few exceptions, slightly worse than all the other estimators.  

We find it surprising that \emph{linearized probit regression} performs equally well as significantly more sophisticated nonlinear estimators on a broad range of real-world datasets. We also note that the proposed linearized estimators can be implemented efficiently and scale well to large datasets, which is in stark contrast to the PM estimator. 
%



\section{Conclusions}

We have shown that linearizing the well-known probit regression model in combination with linear estimators is able to achieve comparable estimation performance to nonlinear methods such as MAP and PM estimators for binary regression problems. Our linear estimators enable an exact, closed-form, and nonasymptotic MSE analysis, which is in stark contrast to existing analytical results for the MAP and PM estimators. 
We hence believe that the proposed linear estimators have the potential to be used in a variety of machine learning or statistics applications that deal with binary-valued observations.  




\balance
\bibliographystyle{IEEEtran} 
\bibliography{confs-jrnls,publishers,studer}

\end{document}